\documentclass{tlp}
\usepackage{aopmath}

\usepackage{amssymb}
\usepackage{url}

\usepackage[polish,english]{babel}
\usepackage[T1,nomathsymbols]{polski}
\usepackage[utf8]{inputenc}
\usepackage[T1]{fontenc}

\usepackage{caption}
\usepackage{subcaption}
\usepackage{wrapfig,lipsum,booktabs}

\usepackage{tikz}
\usetikzlibrary{arrows,backgrounds,automata,topaths,patterns,decorations.pathreplacing,decorations.pathmorphing}
\usetikzlibrary{calc}
\usetikzlibrary{fit}
\usetikzlibrary{arrows}

\usepackage{xcolor}
\newcommand{\OPRAm}[0]{\ensuremath{\mathcal{OPRA}_m}}
\newcommand{\QS}{$\mathcal{QS}$}

\newcommand{\Pred}[1]{ {\footnotesize \texttt{#1}} }

\newcommand{\Const}[1]{\mathrm{#1}}
\newcommand{\Rel}[2]{\mathit{#1} (#2)}
\newcommand{\mysubsecNoBook}[2]{\noindent\textbf{#1}\quad}
\newcommand{\bul}{$\triangleright~$}

\definecolor{DarkBlue}{RGB}{7,72,110}
\definecolor{LightBlue}{RGB}{205,237,249}
\definecolor{LightRed}{RGB}{255,102,102}
\definecolor{DarkRed}{RGB}{153,0,0}
\definecolor{LightGreen}{RGB}{178,205,102}
\definecolor{DarkGreen}{RGB}{0,51,0}

\newtheorem{definition}{Definition} 

\usepackage[pdftitle={Non-Monotonic Spatial Reasoning},
pdfauthor={Prof. Dr. Mehul Bhatt (University of Bremen, Germany.},
linkbordercolor=white,
bookmarksopen=false,
bookmarksnumbered=false]{hyperref}

\submitted{November 2 2015}
\revised{April 11 2016}
\accepted{June 2 2016}

\hypersetup{
    pdftitle={Non-Monotonic Spatial Reasoning}, 
    pdfauthor={Prof. Dr. Mehul Bhatt (University of Bremen, Germany).},     	
    pdfsubject={declarative spatial reasoning}, 
    pdfkeywords= {keyword}, 					
    unicode=false,          									
    pdftoolbar=false,        									
    pdfmenubar=true,        								
    pdffitwindow=false,     									
    pdfstartview={FitH},    								
    pdfnewwindow=true,      								
    bookmarks=true,         								
    colorlinks=true,       									
    linkcolor=blue!80!black,          							
    citecolor=blue!20!black,        							
    filecolor=blue,      									
    urlcolor=blue!50!black          								
}

\begin{document}
\bibliographystyle{acmtrans}

\long\def\comment#1{}

\title[Non-Monotonic Spatial Reasoning]{ Non-Monotonic Spatial Reasoning\\with Answer Set Programming Modulo Theories\footnote{This is an extended version of a paper presented at the Logic Programming and Nonmonotonic Reasoning Conference (LPNMR 2015), invited as a rapid communication in TPLP. The authors acknowledge the assistance of the conference program chairs Giovambattista Ianni and Miroslaw Truszczynski.}\footnote{This paper comes with an online appendix containing Appendices A-H. The online appendix is available via the supplementary materials link from the TPLP web-site.} }

\author[{\sffamily Wa{\l}\k{e}ga, Schultz, Bhatt}]
{{\sffamily Przemys\l{}aw Andrzej Wa{\l}\k{e}ga, Carl Schultz, Mehul Bhatt}\\
Spatial Reasoning. \upshape{\url{www.spatial-reasoning.com}}\\
The DesignSpace Group, Germany. \upshape{\url{www.design-space.org}}\\
$~$\\
Universities of: Warsaw (Poland), M\"{u}nster (Germany), Bremen (Germany)
}

\pagerange{\pageref{firstpage}--\pageref{lastpage}}
\volume{\textbf{10} (3):}
\jdate{Oct 2015}
\setcounter{page}{1}
\pubyear{20xx}

\maketitle

\maketitle

\label{firstpage}

\begin{abstract}
The systematic modelling of \emph{dynamic spatial systems} is a key requirement in a wide range of application areas such as commonsense cognitive robotics, computer-aided architecture design, and dynamic geographic information systems. We present ASPMT(QS), a novel approach and fully-implemented prototype for non-monotonic spatial reasoning ---a crucial requirement within dynamic spatial systems--- based on Answer Set Programming Modulo Theories (ASPMT). 

\smallskip

\noindent ASPMT(QS) consists of a (qualitative) spatial representation module (QS) and a method for turning tight ASPMT instances into Satisfiability Modulo Theories (SMT) instances in order to compute stable models by means of SMT solvers. We formalise and implement concepts of  default spatial reasoning and spatial frame axioms. 
Spatial reasoning is performed by encoding spatial relations as systems of polynomial constraints, and solving via SMT with the theory of real nonlinear arithmetic. We empirically evaluate ASPMT(QS) in comparison with other contemporary spatial reasoning systems both within and outside the context of logic programming. ASPMT(QS) is currently the only existing system that is capable of reasoning about indirect spatial effects (i.e., addressing the ramification problem), and integrating geometric and qualitative spatial information within a non-monotonic spatial reasoning context.

\smallskip

\noindent This paper is under consideration for publication in TPLP.

\end{abstract}


\section{Introduction}
Non-monotonicity is characteristic of commonsense reasoning patterns concerned with, for instance, making default assumptions (e.g., about spatial inertia), counterfactual reasoning with hypotheticals (e.g., what-if scenarios), knowledge interpolation, explanation and diagnosis (e.g., filling the gaps, causal links), and belief revision. Such reasoning patterns, and therefore non-monotonicity, acquires a special significance in the context of \emph{spatio-temporal dynamics}, or computational commonsense \emph{reasoning about space, actions, and change} as applicable within areas as disparate as geospatial dynamics, computer-aided design,  cognitive vision, and commonsense cognitive robotics  \cite{Bhatt:RSAC:2012}. Dynamic spatial systems are characterised by scenarios where spatial configurations of objects undergo a change as the result of interactions within a physical environment \cite{bhatt:scc:08}; this requires explicitly identifying and formalising relevant actions and events at both an ontological and (qualitative and geometric) spatial level; for instance:

\begin{itemize} 

	\item within the context of geographical information systems, formalising high-level processes such as \emph{desertification} and \emph{population displacement} based on spatial theories about \emph{appearance, disappearance, splitting, motion}, and \emph{growth} of regions \cite{Bhatt-Wallgruen-2014};
	
	\item within a domain such as (commonsense) cognitive vision from visual imagery (e.g., video, point-clouds etc), the ability to perform spatio-linguistically grounded semantic question-answering about perceived and hypothesised events based on a commonsense qualitative model of space and motion \cite{ijcai2016-vision-eye-film}.
\end{itemize}

Requirements such as in the aforementioned domains call for a deep integration of spatial reasoning within KR-based non-monotonic reasoning frameworks \cite{Bhatt:RSAC:2012,bhatt2011-scc-trends}.

\medskip

In this paper, we select aspects of a theory of \emph{dynamic spatial systems} ---pertaining to \emph{spatial inertia, ramifications, and causal explanation}--- that are inherent to a broad category of dynamic spatio-temporal phenomena, and require non-monotonic reasoning \cite{bhatt:scc:08,bhatt:aaai08,DBLP:conf/flairs/Bhatt10}. For these aspects, we provide an operational semantics and a computational framework for realising fundamental non-monotonic spatial reasoning capabilities based on Answer Set Programming Modulo Theories \cite{bartholomew2013functional}; ASPMT is extended to the qualitative spatial (QS) domain resulting in the non-monotonic spatial reasoning system ASPMT(QS). Spatial reasoning is performed in an analytic manner (e.g., as with reasoners such as CLP(QS) \cite{bhatt2011clp}), where spatial relations are encoded as systems of polynomial constraints; the task of determining whether a set of qualitative spatial constraints are consistent is now equivalent to determining whether the system of polynomial constraints is satisfiable, i.e., Satisfiability Modulo Theories (SMT) with real nonlinear arithmetic, and can be accomplished in a sound and complete manner. Thus, ASPMT(QS) consists of a (qualitative) spatial representation module and a method for turning tight ASPMT instances into SMT instances in order to compute stable models by means of SMT solvers. 

In the following sections we present the relevant foundations of stable model semantics and ASPMT, and then extend this to ASPMT(QS) by defining a relational, qualitative spatial representation module \QS, and formalising default spatial reasoning and spatial frame and ramification axioms using choice formulas. We empirically evaluate ASPMT(QS) in comparison with other existing spatial reasoning systems. In the backdrop of our results, and a discussion of related work, we conclude that ASPMT(QS) is the only system, to the best of our knowledge, that operationalises dynamic spatial reasoning -- together with a systematic treatment of non-monotonic aspects emanating therefrom, within a KR-based framework.


\section{Spatial Representation and Reasoning}\label{sec:qsr}

Knowledge representation and reasoning about \emph{space} may be formally interpreted within diverse frameworks such as: (a) geometric reasoning and constructive (solid) geometry \cite{Kapur:1989:GR:107262}; (b) relational algebraic semantics of `qualitative spatial calculi' \cite{ligozat-book}; and (c) by axiomatically constructed formal systems of mereotopology and mereogeometry \cite{hdbook-spatial-logics}. Independent of formal semantics, commonsense spatio-linguistic abstractions offer a human-centred and cognitively adequate mechanism for logic-based automated reasoning about spatio-temporal information \cite{Bhatt-Schultz-Freksa:2013}. 


Research in qualitative spatial representation and reasoning has primarily been driven by the use of \emph{relational-algebraic} semantics, and development of constraint-based reasoning algorithms to solve \emph{consistency problems} in the context of \emph{qualitative spatial calculi} \cite{ligozat-book}. The key idea has been to partition an infinite quantity space into finite disjoint categories, and utilize the special \emph{relational algebraic properties} of such a partitioned space for reasoning purposes. Logic-based axiomatisations of topological and mereotopological space, a study of their general computational characteristics from a reasoning viewpoint, and the development of KR-based general reasoning systems have also been thoroughly investigated \cite{hdbook-spatial-logics,bhatt2011clp}.



Different in its foundational method involving the use of  constraint logic programming (CLP) is the \emph{declarative spatial reasoning system} CLP(QS) \cite{bhatt2011clp,schultz-bhatt-2012,DBLP:conf/ecai/SchultzB14}. CLP(QS) marks a clear departure from other relational algebraically founded methods and reasoning tools by its use of the constraint logic programming framework for formalising the semantics of qualitative spatio-temporal relations and formal spatial calculi. CLP(QS) has demonstrated applications in a range of domains including: \emph{architectural design cognition} \cite{DesignNarrative-KR2014}, \emph{cognitive vision} \cite{Bhatt2013-CMN,SuchanBS14-perceptualnarrative,ijcai2016-vision-eye-film,wacv2016-vision-eye-film}, \emph{geospatial dynamics} \cite{Bhatt-Wallgruen-2014}, and \emph{cognitive robotics} \cite{CR-2013-Narra-CogRob,DBLP:conf/pricai/SprangerSBE14,ijcai2016-language-suchan}.

\section{Answer Set Programming Modulo Theories}\label{sec::stable_models}

Stable models semantics is the most expanded theory of non-monotonic reasoning with important practical applications.
It is capable of expressing a number of non-monotonic and default reasoning types, e.g., causal effects of actions, a lack of information and various completeness assumptions (e.g., absence of ramification yielding state constraints) \cite{gelfond2008answer}, which makes it very attractive for \emph{reasoning about space, actions, and change} in an integrated manner. In what follows, we present a definition of stable models based on syntactic transformations \cite{bartholomew2012stable}  which is a generalization of previous definitions from \cite{ferraris2011stable}, \cite{gelfond1988stable}, and \cite{ferraris2005answer}. We then present a method for turning tight ASPMT instances into SMT instances.

\subsection{Bartholomew -- Lee Functional Stable Model Semantics}

In what follows we adopt a definition of stable models based on syntactic transformations presented in \cite{bartholomew2012stable}.
For predicate symbols (constants or variables) $u$ and $c$, expression $u \leq c$ is defined as shorthand for $\forall \textbf{x}(u(\textbf{x}) \to c(\textbf{x}))$. Expression $u = c$ is defined as $\forall \textbf{x}(u(\textbf{x}) \equiv c(\textbf{x}))$ if $u$ and $c$ are predicate symbols, and $\forall \textbf{x}(u(\textbf{x}) = c(\textbf{x}))$ if they are function symbols. For lists of symbols $\textbf{u} = (u_1, \dots , u_n )$ and $\textbf{c} = (c_1, \dots , c_n )$, expression $\textbf{u} \leq \textbf{c}$ is defined as $(u_1 \leq c_1) \land \dots \land (u_n \leq c_n )$, and similarly, expression $\textbf{u} = \textbf{c}$ is defined as $(u_1 = c_1) \land \dots \land (u_n = c_n )$. Let $\textbf{c}$ be a list of distinct predicate and function constants, and let $\widehat{\textbf{c}}$ be a list of distinct predicate and function variables corresponding to c. By $\textbf{c}^{pred}$ ($\textbf{c}^{func}$ , respectively) we mean the list of all predicate constants (function constants, respectively) in $\textbf{c}$, and by $\widehat{\textbf{c}}^{pred}$ ($\widehat{\textbf{c}}^{func}$ , respectively) the list of the corresponding predicate variables (function variables, respectively) in $\widehat{\textbf{c}}$. We refer to function constants and predicate constants of arity $0$ as object constants and propositional constants, respectively.

\begin{definition}[\textbf{Stable model operator $\textbf{SM}$}] \label{def::stable_operator}
For any formula $F$ and any list of predicate and function constants $\textbf{c}$ (called intensional constants), $\textbf{SM}[F;\textbf{c}]$ is defined as
$$
F \land \neg \exists \widehat{\textbf{c}} ( \widehat{\textbf{c}} < \textbf{c} \land F^*(\widehat{\textbf{c}}) ),
$$
where $\widehat{\textbf{c}} < \textbf{c}$ is a shorthand for $(\widehat{\textbf{c}}^{pred} \leq \textbf{c}^{pred}) \land \neg (\widehat{\textbf{c}} = \textbf{c})$ and $F^*(\widehat{\textbf{c}})$ is defined recursively as follows:
\begin{itemize}
\item for atomic formula $F$, $F^* \equiv F' \land F$, where $F'$ is obtained from $F$ by replacing all intensional constants $\textbf{c}$ with corresponding variables from $\widehat{\textbf{c}}$,
\item $(G \land H)^* = G^* \land H^*$, \ \ \ $(G \lor H)^* = G^* \lor H^*$,	   
\item $(G \to H)^* = (G^* \to H^*) \land (G \to H)$,
\item $(\forall x G)^* = \forall x G^*$, \ \ \ $(\exists x G)^* = \exists x G^*$.	 
\end{itemize}
$\neg F$ is a shorthand for $F \to \bot$, $\top$ for $\neg \bot$ and $F \equiv G$ for $(F \to G) \land (G \to F)$.
\end{definition}

\begin{definition}[\textbf{Stable model}] \label{def::stable_model}
For any sentence $F$, a stable model of $F$ on $\textbf{c}$ is an interpretation $I$ of the underlying signature such that $I \models \textbf{SM}[F;\textbf{c}]$.
\end{definition}

\subsection{Turning ASPMT into SMT}

An SMT instance is a formula in a many-sorted first-order logic, i.e., with fixed meaning (by a background theory) of designated functions and predicates constants. Then, the SMT problem is to check if a given SMT instance has a model that expands the background theory. ASPMT is an generalization of ASP analogous to a generalization of SAT obtained with SMT (the syntax of ASPMT is the same as the syntax of SMT).

It is shown in \cite{bartholomew2013functional} that a tight part of ASPMT instances can be turned into SMT instances and, as a result, off-the-shelf SMT solvers (e.g., \textsc{Z3} \cite{de2008z3} for arithmetic over reals) may be used to compute stable models of ASP.
In order to capture this statement formally in Theorem \ref{the::smt} we firstly introduce notions of Clark normal form, Clark completion and dependency graphs. 

\begin{definition}[\textbf{Clark normal form}] \label{def::clark}
Formula $F$ is in \emph{Clark normal form} (relative to the list $\textbf{c}$ of intensional constants) if it is a conjunction of sentences of the form (\ref{eqn::clark1}) and (\ref{eqn::clark2})

\begin{equation} \label{eqn::clark1}
\forall \textbf{x} (G \to p(\textbf{x})),
\end{equation}
\begin{equation} \label{eqn::clark2}
\forall \textbf{x}y (G \to f(\textbf{x})=y)\mbox{,}
\end{equation}
%

\noindent one for each intensional predicate $p$ and each intensional function $f$, where $\textbf{x}$ is a list of distinct object variables, $y$ is an object variable, and $G$ is an arbitrary formula that has no free variables other than those in $\textbf{x}$ and $y$.
\end{definition}

\begin{definition}[\textbf{Clark completion}] \label{def::completion}
The \emph{completion} of a formula $F$ in Clark normal form (relative to $\textbf{c}$), denoted by $Comp_{\textbf{c}}[F]$ is obtained from  $F$ by replacing each conjunctive term of the form (\ref{eqn::clark1}) and (\ref{eqn::clark2}) with (\ref{eqn::completion1}) and (\ref{eqn::completion2}) respectively

\begin{equation} \label{eqn::completion1}
\forall \textbf{x} (G \equiv p(\textbf{x})),
\end{equation}

\begin{equation} \label{eqn::completion2}
\forall \textbf{x}y (G \equiv f(\textbf{x})=y)\mbox{.}
\end{equation}

\end{definition}

\begin{definition}[\textbf{Dependency graph}] \label{def::dependency}
The \emph{dependency graph} of a formula $F$ (relative to $\textbf{c}$) is a directed graph $DG_{\textbf{c}}[F]=(V,E)$ such that:
\begin{enumerate}
\item $V$ consists of members of $\textbf{c}$,
\item for each $c,d \in V$, $(c,d) \in E$ whenever there exists a strictly positive occurrence of $G\to H$ in $F$, such that $c$ has a strictly positive occurrence in $H$ and $d$ has a strictly positive occurrence in $G$,
\end{enumerate}
where an occurrence of a symbol or a subformula in $F$ is called strictly positive in $F$ if that occurrence is not in the antecedent of any implication in $F$.
\end{definition}

\begin{definition}[\textbf{Tight Formula}] \label{def::tight}
Formula $F$ is \emph{tight} (on $\textbf{c}$) if $DG_{\textbf{c}}[F]$ is acyclic.
\end{definition}

The result obtained by Bartholomew and Lee is stated in the following theorem.

\begin{theorem}[\cite{bartholomew2013functional}] \label{the::smt}
For a sentence $F$ in Clark normal form that is tight on $\textbf{c}$, an interpretation $I$ that satisfies $\exists xy (x \neq y)$ is a model of $\textbf{SM}[F;\textbf{c}]$ iff $I$ is a model of $Comp_{\textbf{c}}[F]$ relative to~$\textbf{c}$.
\end{theorem}

%
%
%


\section{ASPMT(QS) -- ASPMT with Qualitative Space (\QS)}\label{sec::aspmt(qs)}


In this section we present our spatial extension of ASPMT, and formalise spatial default rules and spatial frame axioms.

\subsection{The Qualitative Spatial Domain \QS}\label{sec:qs-in-spmt}
Qualitative spatial calculi can be classified into two groups: topological and positional calculi. With topological calculi such as the \emph{Region Connection Calculus} (RCC) \cite{randell1992spatial}, the primitive entities are spatially extended regions of space, and could possibly even be 4D spatio-temporal histories, e.g., for \emph{motion-pattern} analyses. Alternatively, within a dynamic domain involving translational motion, point-based abstractions with orientation calculi could suffice (e.g., using the Oriented-Point Relation Algebra ({\small\OPRAm}) \cite{moratz06_opra-ecai}). The qualitative spatial domain (\QS)  that we consider in the formal framework of this paper encompasses the following ontology.


\noindent\mysubsecNoBook{{\color{black}QS1.\quad Domain Entities in \QS}}{XX}  Domain entities in \QS $\;$ include \emph{circles, triangles, points}, \emph{segments}, \emph{convex polygons}, and \emph{egg-yolk regions}. While our method is applicable to a wide range of 2D and 3D spatial objects and qualitative relations for example, as defined in \cite{pesant1994quad,bouhineau1996solving,pesant1999reasoning,bouhineau1999application,bhatt2011clp,schultz-bhatt-2012,DBLP:conf/ecai/SchultzB14}, for brevity and clarity we primarily focus on a 2D spatial domain: 

\begin{itemize}
	\item a \emph{point} is a pair of reals $x,y$,
	\item a \emph{line segment} is a pair of end points $p_1, p_2$ ($p_1 \neq p_2$),
	\item a \emph{circle} is a centre point $p$ and a real radius $r$ ($0 < r$),
	\item a \emph{triangle} is a triple of vertices (points) $p_1, p_2, p_3$ such that $p_3$ is placed to the \emph{left} of the directed segment $p_1,p_2$, i.e. left with respect to the direction of the segment from $p_1$ to $p_2$,
	\item a \emph{convex polygon} is defined by a list of $n$ vertices (points) $p_1, \dots, p_n$ (spatially ordered counter-clockwise) such that $p_k$ is \emph{left of} the directed segment $p_i, p_j$ for all $1 \leq i < j < k \leq n$,
	\item an \emph{egg yolk} region\footnote{The egg-yolk method of modelling regions with indeterminante boundaries \cite{cohn1996egg} can be employed to characterise a class of regions (including polygons) that satisfies topological and relative orientation relations \cite{schultz_bhatt_lqmr2015}. Each egg-yolk region is an equivalence class for all regions that are contained within the upper approximation (the \emph{egg white}), and completely contain the lower approximation (the \emph{egg yolk}). 
} is defined by a circular upper and lower approximation $c^{+}, c^{-}$ such that $c^{-}$ is a \emph{proper part} of $c^{+}$.
\end{itemize}



\noindent\mysubsecNoBook{{\color{black}QS2.\quad Spatial Relations in \QS}}{XX} 
We define a range of spatial relations with the corresponding polynomial encodings. Examples of spatial relations in \QS $\;$ include:

\noindent \emph{Relative Orientation.}\quad \emph{Left, right, collinear} orientation relations between \emph{points} and \emph{segments}, and \emph{parallel, perpendicular} relations between \emph{segments} \cite{leecomplexity-ecai2014}.

\smallskip

\noindent \emph{Mereotopology.}\quad \emph{Part-whole} and \emph{contact} relations between regions \cite{varzi1996parts,randell1992spatial}.

\subsection{Spatial representations in ASPMT(QS)}

Spatial representations in ASPMT(QS) are based on parametric functions and qualitative relations, defined as follows.

\begin{definition}[\textbf{Parametric function}] \label{def::parametric}
A \emph{parametric function} is an $n$--ary function 
$$
f_n:D_1 \times D_2 \times \dots \times D_n \to \mathbb{R},
$$
such that for any  $i \in \{ 1 \dots n\}$, $D_i$ is a type of spatial object, e.g., $Points$, $Circles$, $Polygons$, etc. 
\end{definition}

As an example consider the following parametric functions $$x:Circles \to \mathbb{R},$$ $$y:Circles \to \mathbb{R},$$ $$r:Circles \to \mathbb{R},$$ which return the position values $x, y$ of a circle's centre and its radius $r$, respectively. Then, circle $c \in Cirlces$ may be described by means of parametric functions as follows: 
$$x(c)=1\mbox{.}23 \land y(c)=-0\mbox{.}13 \land r(c)=2 \textrm{.}$$

\begin{definition}[\textbf{Qualitative spatial relation}] \label{def::qualitative}
A \emph{qualitative spatial relation} is an $n$-ary predicate
$$
Q_n \subseteq D_1 \times D_2 \times \dots \times D_n,
$$
such that for any  $i \in \{ 1 \dots n\}$, $D_i$ is a type of spatial object. For each $Q_n$ there is a corresponding formula of the form
$$
\forall d_1 \in D_1 \dots \forall d_n \in D_n \bigg(  p_1(d_1, \dots , d_n) \land \dots \land p_m(d_1, \dots , d_n) \to Q_n(d_1, \dots , d_n)
  \bigg),
$$
where $m \in \mathbb{N}$ and for any  $i \in \{ 1, \dots, m\}$, $p_i$ is a polynomial equation or inequality over parametric functions involving $d_1, \dots, d_n$. 
\end{definition}


As an example consider a qualitative spatial relation $EC \subseteq Circles \times Circles$, informally interpreted as ``two circles are externally connected''. The corresponding formula of the relation states that two circles are in the $EC$ relation whenever the distance between their centers is equal to the sum of their radii and is as follows:  
$$
\forall c_1,c_2 \in Circles \bigg(  
( (x(c_1)-x(c_2))^2+(y(c_1)-y(c_2))^2 = (r(c_1)+r(c_2))^2
\to EC(c_1,c_2)
  \bigg),
$$

\begin{proposition}\label{prop::qual_clark}
Each qualitative spatial relation according to Definition~\ref{def::qualitative} may be represented as a tight formula in Clark normal form.
\end{proposition}

\begin{proof}
Follows directly from Definitions~\ref{def::clark} and \ref{def::qualitative}.
\end{proof}

Thus, qualitative spatial relations belong to a part of ASPMT that may be turned into SMT instances by transforming the implications in the corresponding formulas into equivalences (Clark completion). The obtained equivalence between polynomial expressions and predicates enables us to compute relations whenever parametric information is given, and vice versa, i.e., computing possible parametric values when only the qualitative spatial relations are known.

Many relations from existing qualitative calculi may be represented in ASPMT(QS) according to Definition~\ref{def::qualitative}; our system can express the polynomial encodings presented in, e.g., \cite{pesant1994quad,bouhineau1996solving,pesant1999reasoning,bouhineau1999application,bhatt2011clp}. In what follows we give some illustrative results (see Appendix B for proofs).
\begin{proposition}\label{prop::IA}
Each relation of Interval Algebra (IA) \cite{allen1983maintaining} and Rectangle Algebra (RA) \cite{guesgen1989spatial} may be defined in ASPMT(QS).
\end{proposition}
\begin{proposition}\label{prop::LR}
Each relation of the Left-Right Algebra (LR) \cite{Scivos2004} may be defined in ASPMT(QS).
\end{proposition}
\begin{proposition}\label{prop::RCC-5}
Each relation of RCC--5 \cite{randell1992spatial} in the domain of convex polygons with a finite number of vertices may be defined in ASPMT(QS).
\end{proposition}
\begin{proposition}\label{prop::CDC}
Each relation of Cardinal Direction Calculus (CDC) \cite{frank1991qualitative} may be defined in ASPMT(QS).
\end{proposition}

\subsection{Polynomial Semantics for \QS}\label{sec:poly-qs}
\begin{wrapfigure}{r}{0.35\textwidth}
  \begin{center}
  \vspace{-25pt}
\includegraphics[width=0.2\textwidth]{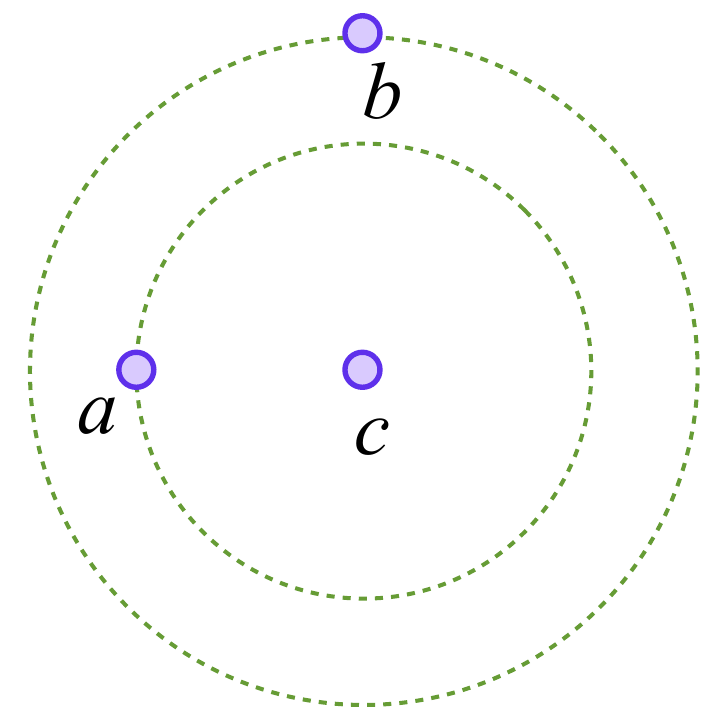}
  \end{center}
    \vspace{-15pt}
\caption{$a$ is \emph{nearer than} $b$ with respect to $c$.}\label{fig:dist-example}
  \vspace{-10pt}
\end{wrapfigure}

Analytic geometry gives us a general, sound, and complete way of working with spatial relations.\footnote{Tarski famously proved that the theory of real-closed fields is decidable via quantifier elimination (see \cite{collins1975quantifier,arnon1984cylindrical,Collins1991299} for an overview and algorithms); i.e., in a finite amount of time we can determine the consistency (or inconsistency) of any formula consisting of quantifiers ($\forall$, $\exists$) over the reals, and polynomial equations and inequalities combined using logical connectors ($\wedge$, $\vee$, $\neg$). Thus, by encoding spatial relations as systems of polynomial constraints (i.e., analytic geometry) we can employ polynomial constraint solving methods that are guaranteed to determine (in)consistency, giving us sound and complete spatial reasoning.} Analytic geometry can be applied to encode the semantics of high-level qualitative spatial relations using systems of real polynomial constraints; determining spatial consistency is then equivalent to determining satisfiability of these polynomial constraints. Given polynomial constraints over a set of real variables X, the constraints are satisfiable if there exists some real value for each variable in X such that all the polynomial constraints are simultaneously satisfied. As an example, consider the definition of a relative qualitative distance relation such as \emph{nearer than}: point $a$ is \emph{nearer than} $b$ with respect to a reference point $c$ if the distance between $a,c$ is less than the distance between $b,c$, denoted $\Rel{nearer\_than}{a,b,c}$ (Figure \ref{fig:dist-example}): 
$$ (a_x - c_x)^2 + (a_y - c_y)^2 < (b_x - c_x)^2 + (b_y - c_y)^2 .$$

Similarly, the relation of point $p_3$ being \emph{left of} segment $s_{p_1 p_2}$ is encoded as the following polynomial constraint \cite{bhatt2011clp}:
$$ p_3 ~ \text{\emph{left of}} ~ s_{p_1 p_2} \equiv_{def} (x_2 - x_1) (y_3 - y_1) > (y_2 - y_1) (x_3 - x_1) $$

If there exists an assignment of real values to the variables (e.g., $a_x = 3, c_x = 10.5$ in the nearness relation) that satisfies the polynomial equations and inequalities, then the qualitative spatial relations are consistent. Continuing with the example of relative nearness relations, consider a qualitative spatial description with the relations: $\Rel{nearer\_than}{a,b,c}$, $\Rel{nearer\_than}{b,a,c}$. This is encoded in the following polynomial constraints:
$$ (a_x - c_x)^2 + (a_y - c_y)^2 < (b_x - c_x)^2 + (b_y - c_y)^2, $$
$$(a_x - c_x)^2 + (a_y - c_y)^2 > (b_x - c_x)^2 + (b_y - c_y)^2, $$

which can be reformulated as
$$d_{ac} < d_{bc}, ~~d_{ac} > d_{bc}\textrm{.}$$

As the real value $d_{ac}$ cannot be both greater and smaller than $d_{bc}$, this system of polynomial constraints is inconsistent, and no configuration of points (within a Euclidean space) exists that can satisfy this set of qualitative spatial constraints. 

A range of qualitative spatial relations can be similarly encoded in the form of polynomial constraints, e.g., see Table \ref{tab:rcc-encodings} and Figure \ref{fig:rcc} for encondings for RCC relations\footnote{The region connection calculus (RCC) is a spatial logic of topological relations between regions \cite{randell1992spatial}. The theory is based on a single \emph{connects} predicate $C(x,y)$ interpreted as the topological closures of regions $x$ and $y$ having at least one point in common (i.e., the regions \emph{touch} at their boundaries or their interiors \emph{overlap}).} with corresponding ASPMT encodings in Appendix C. 

\begin{table}[htp]
 \caption{{\small Polynomial encodings of RCC relations between circles $c_1,c_2$ (omitting inverses), where $x_i = x(c_i)$, $y_i = y(c_i)$, $r_i = r(c_i)$, and $\Delta(c_1, c_2) = (x_1 - x_2)^2+(y_1 - y_2)^2$.}}
\centering
\scriptsize
\begin{tabular}{p{50 ex}lll}
\hline
\hline
\textbf{\footnotesize RCC Relation} & \textbf{\footnotesize Polynomial Encoding} \\
\hline
contact (C)  & $\Delta(c_1, c_2) \leq (r_1 + r_2)^2$ \\
discrete from (DR)  & $\Delta(c_1, c_2) \geq (r_1 + r_2)^2$ \\
disconnects (DC)  & $\Delta(c_1, c_2) > (r_1 + r_2)^2$ \\
externally connects (EC)  & $\Delta(c_1, c_2) = (r_1 + r_2)^2$ \\
overlaps (O)  & $\Delta(c_1, c_2) < (r_1 + r_2)^2$ \\
partially overlaps (PO)  & $(r_1 - r_2)^2 < \Delta(c_1, c_2) < (r_1 + r_2)^2$ \\
part of (P)  & $\Delta(c_1, c_2) \leq (r_1 - r_2)^2 \wedge (r_1 \leq r_2)$ \\
proper part of (PP)  & $\Delta(c_1, c_2) \leq (r_1 - r_2)^2 \wedge (r_1 < r_2)$ \\
tangential proper part (TPP)  & $\Delta(c_1, c_2) = (r_1 - r_2)^2 \wedge (r_1 < r_2)$ \\
nontangential proper part (NTPP)  & $\Delta(c_1, c_2) < (r_1 - r_2)^2 \wedge (r_1 < r_2)$ \\
equal (EQ)  & $x_1 = x_2 \wedge y_1 = y_2 \wedge r_1 = r_2$ \\
\hline
\hline
\end{tabular}
\label{tab:rcc-encodings}
\end{table}

\begin{wrapfigure}{r}{0.45\textwidth}
  \begin{center}
  \vspace{-20pt}
\includegraphics[width=0.35\textwidth]{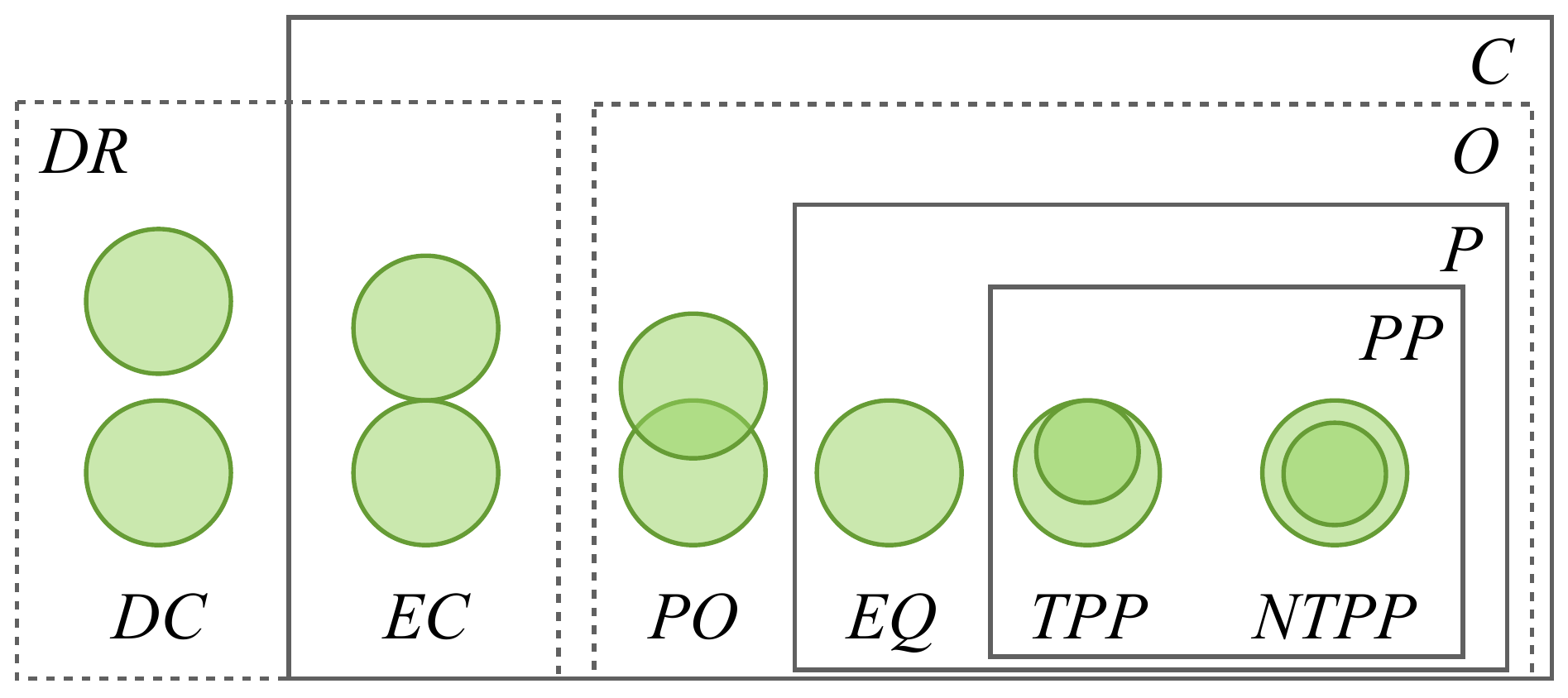}
  \end{center}
    \vspace{-12pt}
\caption{RCC relations between two circular regions.}
\label{fig:rcc}
\end{wrapfigure}
\vspace{-10pt}
Examples of specific polynomial encodings for topological and orientation relations can be found in \cite{pesant1994quad,bouhineau1996solving,pesant1999reasoning,bhatt2011clp,DBLP:conf/ecai/SchultzB14,DBLP:conf/cosit/SchultzB15} (also, Appendix H presents optimisations for encodings). Since ASPMT(QS) enables the use of polynomial equations and inequalities, a number of spatial relations known from qualitative calculi may be expressed (Propositions \ref{prop::IA}-\ref{prop::CDC}, and Appendix B).

\section{ASPMT(QS) -- System Overview and Implementation}\label{sec::implementation}

The implementation of ASPMT(QS) builds on \textsc{aspmt2smt} \cite{bartholomew2014system} --  a compiler translating a tight fragment of ASPMT into SMT instances. Our system consists of an additional module for spatial reasoning and \textsc{Z3} as the SMT solver. As our system operates on a tight fragment of ASPMT, input programs need to fulfil certain requirements, described in the following section. As output, our system either produces the stable models of the input programs, or states that no such model exists.

\subsection{Syntax of Input Programs}

The input program to our system needs to be $f$-$plain$ to use Theorem 1 from \cite{bartholomew2012stable}.

\begin{definition}[\textbf{$\textbf{f-plain}$ formula}]
Let $f$ be a function constant. A first--order formula is called $f$-$plain$ if each atomic formula:
\begin{itemize}
\item does not contain $f$, or
\item is of the form $f(\textbf{t}) = u$, where $\textbf{t}$ is a tuple of terms not containing $f$, and $u$ is a term not containing $f$.
\end{itemize}
\end{definition}

Additionally, the input program needs to be \emph{av-separated}, i.e. no variable occurring in an argument of an uninterpreted function is related to the value variable of another uninterpreted function via equality \cite{bartholomew2014system}. The input program is divided into declarations of:
\begin{itemize}
\item $\Pred{sorts}$ (data types),
\item $\Pred{objects}$ (particular elements of given types),
\item $\Pred{constants}$ (functions),
\item $\Pred{variables}$ (variables associated with declared types).
\end{itemize}
The second part of the program consists of clauses. ASPMT(QS) supports:
\begin{itemize}
\item connectives: $\Pred{\&}$, $\Pred{|}$, $\Pred{not}$, $\Pred{->}$, $\Pred{<-}$, 
\item arithmetic operators: \texttt{<}, \texttt{<=}, \texttt{>=}, \texttt{>}, \texttt{=}, \texttt{!=}, \texttt{+}, \texttt{=}, \texttt{*}, 
\end{itemize}
with their usual meaning. Additionally, ASPMT(QS) supports the following as native / first-class entities:

\begin{itemize}
\item basic spatial objects types, e.g., $\Pred{point}$, $\Pred{segment}$, $\Pred{circle}$, $\Pred{triangle}$;
\item parametric functions describing objects parameters, e.g., $x(\Pred{point})$, $r(\Pred{circle})$;
\item qualitative relations, e.g., $\Pred{rccEC}(\Pred{circle},\Pred{circle})$, $\Pred{coincident}(\Pred{point},\Pred{circle})$.
\end{itemize}

The abovementioned qualitative spatial relations and parametric functions do not need to be defined in the program (they are predefined in our spatial reasoning module). Similarly, object types do not need to be defined. The user only needs to define names of objects and the type they belong to, e.g., ``a :: circle'', stands for an object ``a'' that is a circle. 

\noindent\bul\textbf{Example 1: combining topology and size}\quad  Consider a program describing three circles $a$, $b$, $c$ such that $a$ is discrete from $b$,  $b$ is discrete from $c$, and $a$ is a proper part of $c$, declared as follows:

\includegraphics[]{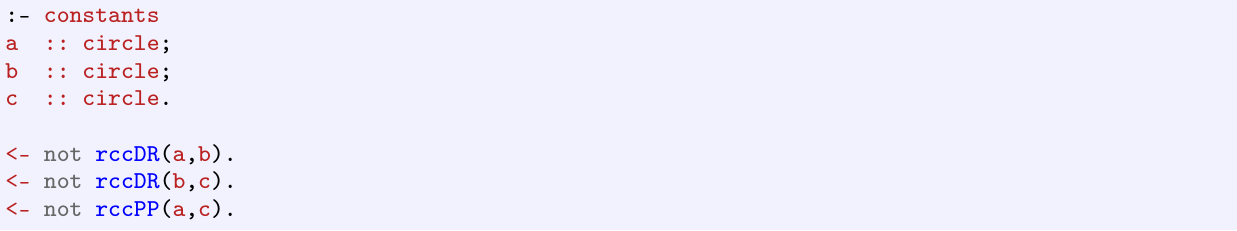}
%

\noindent ASPMT(QS) checks if the spatial relations are satisfiable. In the case of a positive answer, a parametric model and computation time are presented. The above mentioned program generates the following output:

\includegraphics[]{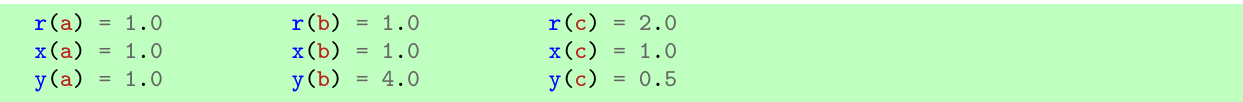}

\noindent This example demonstrates that ASPMT(QS) is capable of computing whether a given set of qualitative relations is consistent. In this case we show that ASPMT(QS) may compute RCC--5 relations between circles (for details see \cite{randell1992spatial}). Now, consider the addition of a further constraint to the program stating that circles $a$, $b$, $c$ have the same radius:

\includegraphics[]{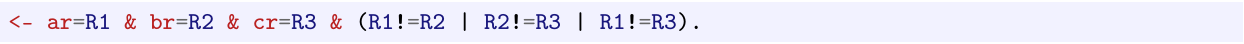}

\noindent The new program is an example of combining different \emph{types} of qualitative information, namely topology and size, which is a non-trivial research topic within the relation algebraic spatial reasoning community; relation algebraic-based solvers such as GQR \cite{gantner2008gqr,DBLP:conf/ijcai/WolflW09} will not correctly determine inconsistencies in general for arbitrary combinations of different types of relations (orientation, shape, distance, etc.). In this case, ASPMT(QS) correctly determines that the spatial constraints are inconsistent:

\includegraphics[]{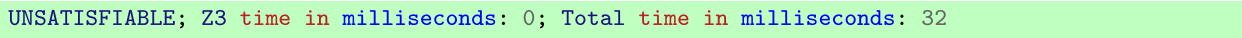}


\noindent\bul\textbf{Example 2: combining topology and relative orientation}\quad  Given three circles $a$, $b$, $c$ let $a$ be proper part of $b$,  $b$ discrete from $c$, and $a$ in contact with $c$, declared as follows: 

\includegraphics[]{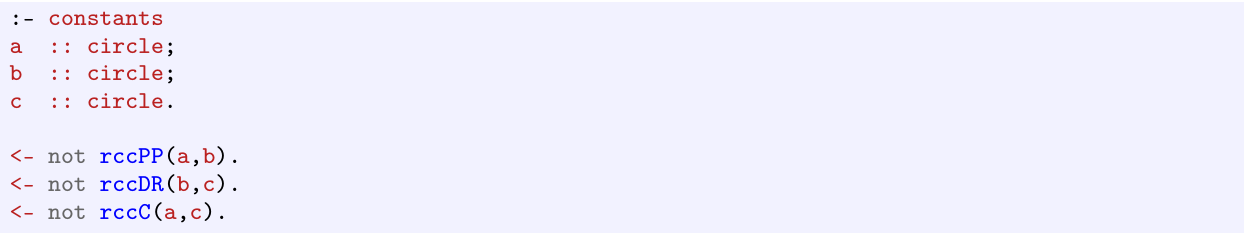}

%

\noindent Given this basic qualitative information, ASPMT(QS) is able to refine the topological relations to infer that (Figure~\ref{fig:topo-ori-1}):
\begin{itemize}
\item $a$ is a \emph{tangential proper part} of $b$
\item both $a$ and $b$ are \emph{externally connected} to $c$
\end{itemize}
as stated in the following output:

\includegraphics[]{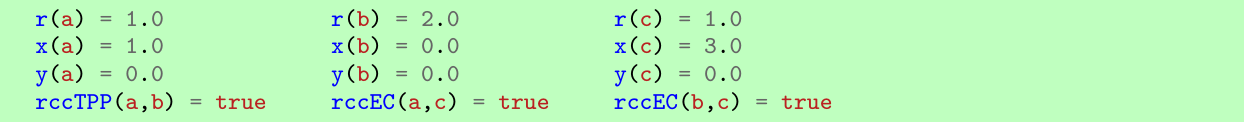}


\noindent We then add an additional constraint that the centre of $a$ is \emph{left of} the directed segment between the centres $b$ to $c$. 

\includegraphics[]{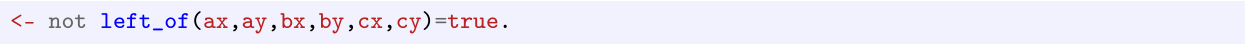}


\noindent ASPMT(QS) determines that this is inconsistent:

\includegraphics[]{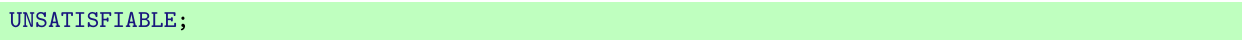}


\noindent In other words the centres must be \emph{collinear} (therefore the configuration presented in Figure \ref{fig:topo-ori-2} is not possible).

\begin{figure}[ht]
    \centering
    \begin{subfigure}[b]{0.3\textwidth}
        \centering
        \includegraphics[width=0.7\textwidth]{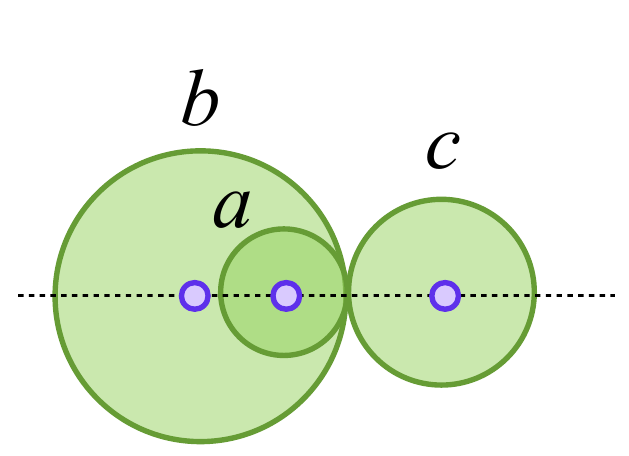}
        \caption{}
        \label{fig:topo-ori-1}
    \end{subfigure}%
    \begin{subfigure}[b]{0.3\textwidth}
        \centering
        \includegraphics[width=0.7\textwidth]{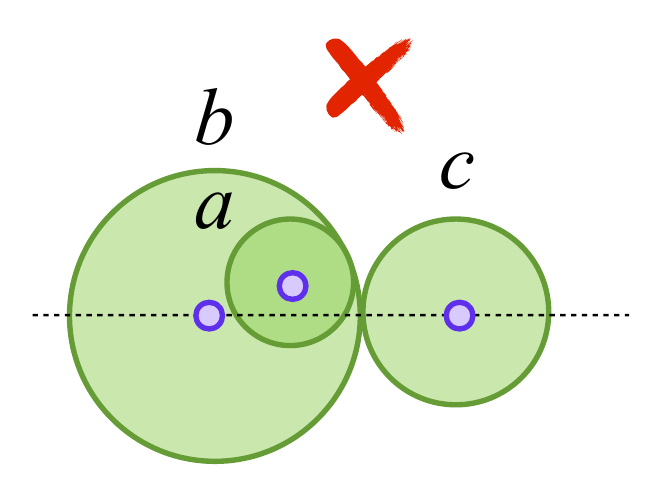}
        \caption{}
        \label{fig:topo-ori-2}
    \end{subfigure}
    \caption{Reasoning about consistency and refinement by combining topology and relative orientation.}
    \label{fig:topo-ori}
    \vspace{-10pt}
\end{figure}

\section{Empirical Evaluation and Examples} \label{sec::tests}

In this section we present an empirical evaluation of ASPMT(QS) in comparison with other existing spatial reasoning systems. The range of problems demonstrates the unique, non-monotonic spatial reasoning features that ASPMT(QS) provides beyond what is possible using other currently available systems.
Table~\ref{tab::cumulative} presents run times obtained by Clingo -- an ASP grounder and solver \cite{gebser2014clingo}, GQR -- a binary constraint calculi reasoner \cite{gantner2008gqr}, CLP($\mathcal{QS}$) -- a Constraint Logic Programming system over qualitative spatial domains \cite{bhatt2011clp}\footnote{The fundamental distinction between CLP($\mathcal{QS}$) and ASPMT(QS) is the declarative framework that has been extended, in particular, Constraint Logic Programming does not facilitate non-monotonic reasoning and thus cannot be used to formalise default rules, spatial inertia, etc.} and our ASPMT(QS) implementation. Tests were performed on an Intel Core 2 Duo 2.00 GHZ CPU with 4 GB RAM running Ubuntu 14.04. The polynomial encodings of topological relations are presented in Table~\ref{tab:rcc-encodings}.



\begin{table}[h]
  \vspace{-0pt}
\footnotesize
\begin{center}
\caption{\textit{\small Total time results (i.e., preprocessing and solving) of performed tests. ``---'' indicates that the problem can not be formalised, ``$^{I}$'' indicates that indirect effects can not be formalised, ``$^{D}$'' indicates that default rules can not be formalised.}
}
\label{tab::cumulative}
\begin{tabular}{crrrr}
\hline
\hline
Problem & \multicolumn{1}{c}{Clingo} & \multicolumn{1}{c}{GQR} & \multicolumn{1}{c}{CLP($\mathcal{QS}$)} & \multicolumn{1}{c}{ASPMT(QS)}  \\
\hline
\emph{Growth} & 0.004s$^{I}$ & 0.014s$^{I,D}$ & 1.623s$^{D}$ & 0.169s  \\
\emph{Motion}  & 0.004s$^{I}$ & 0.013s$^{I,D}$ &  0.449s$^{D}$  & 0.167s  \\
\emph{Attach I}  & 0.008s$^{I}$ & \multicolumn{1}{c}{---} & 3.139s$^{D}$  & 0.625s  \\
\emph{Attach II} & \multicolumn{1}{c}{---}  & \multicolumn{1}{c}{---} & 2.789s$^{D}$  & 0.268s  \\
\hline
\hline
\end{tabular}
\end{center}
  \vspace{-20pt}
\end{table}

\subsection{Ramification Problem}\label{sec::ramification}
The following two problems, \emph{Growth} and \emph{Motion}, were introduced in \cite{bhatt:aaai08}. Consider the initial situation $S_0$ presented in Figure~\ref{fig::ramifications}, consisting of three cells: $a$, $b$, $c$, such that $a$ is a non-tangential proper part of $b$: $\Pred{rccNTPP}(a,b,0)$, and $b$ is externally connected to $c$: $\Pred{rccEC}(b,c,0)$. 

\begin{figure}[ht]
  \begin{center}
\resizebox{0.8\textwidth}{!}{
\begin{tikzpicture}

\node[draw=none,fill=none] at (-2.5,0.7)  {$S_0:$};
\node[draw=none,fill=none] at (-2.5,-1.3)  {$S_1:$};

\node (0) at (1,0.7) [line width=1.2pt,rounded corners=2pt, draw,thick,minimum width=2cm,minimum height=1.4cm] {};
\node[draw=black, circle, minimum size=11mm, inner sep=0pt,outer sep=0] at (0.7,0.7) {$ $};
\node[draw=black, fill=LightBlue, circle, minimum size=5mm, inner sep=0pt,outer sep=0] at (0.7,0.8) {$a$};
\node[draw=black, fill=LightGreen, circle, minimum size=6mm, inner sep=0pt,outer sep=0] at (1.55,0.7) {$c$};
\node[draw=none,fill=none] at (0.7,0.35)  {$b$};

\node (1) at (-0.5,-1.3) [line width=1.2pt,rounded corners=2pt, draw,thick,minimum width=2cm,minimum height=1.4cm] {};
\node[draw=black, fill=LightBlue, circle, minimum size=11mm, inner sep=0pt,outer sep=0] at (-0.8,-1.3) {$a=b$};
\node[draw=black, fill=LightGreen, circle, minimum size=6mm, inner sep=0pt,outer sep=0] at (0.05,-1.3) {$c$};

\node (or1) at (2,-1.3) [line width=1.2pt,rounded corners=2pt, draw,thick,minimum width=2cm,minimum height=1.4cm] {};
\node[draw=black, circle, minimum size=11mm, inner sep=0pt,outer sep=0] at (1.7,-1.3) {};
\node[draw=black, fill=LightBlue, circle, minimum size=5mm, inner sep=0pt,outer sep=0] at (1.7,-1.0)  {$a$};
\node[draw=black, fill=LightGreen, circle, minimum size=6mm, inner sep=0pt,outer sep=0] at (2.55,-1.3) {$c$};
\node[draw=none,fill=none] at (1.7,-1.65)  {$b$};

\node (or2) at (4.5,-1.3) [line width=1.2pt,rounded corners=2pt, draw,thick,minimum width=2cm,minimum height=1.4cm] {};
\node[draw=black, circle, minimum size=11mm, inner sep=0pt,outer sep=0] at (4.2,-1.3) {};
\node[draw=black, fill=LightBlue, circle, minimum size=5mm, inner sep=0pt,outer sep=0] at (4.5,-1.3)  {$a$};
\node[draw=black, fill=LightGreen, circle, minimum size=6mm, inner sep=0pt,outer sep=0] at (5.05,-1.3) {$c$};
\node[draw=none,fill=none] at (3.9,-1.3)  {$b$};

\node[dashed, rounded corners=2pt, fit=(or1)(or2), draw] {};

\draw[->, line width=1.2pt, rounded corners=2pt] (0,0.7) -- (-0.75,0.7) -- (-0.75,-0.6);
\draw[ line width=1.2pt, rounded corners=2pt] (2,0.7) -- (3.25,0.7) -- (3.25,0.3);

\draw[->, line width=1.2pt, rounded corners=2pt] (3.25,0.3) -- (2.25,-0.4);
\draw[->, line width=1.2pt, rounded corners=2pt] (3.25,0.3) -- (4.25,-0.4);
\node[draw=none,fill=none] at (3.25,-0.1)  {OR};

\node at (-1.75,0.2) [] {$growth(a,0)$};
\node at (4.5,0.2) [] {$motion(a,0)$};







\end{tikzpicture}
}
\caption{Indirect effects of $growth(a,0)$ and $motion(a,0)$ events.}
\label{fig::ramifications}
\end{center}
\end{figure}
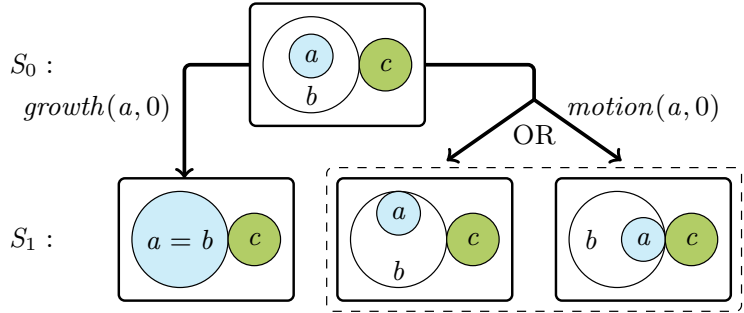






$\triangleright$ \emph{Growth.} Let $a$ grow in step $S_0$; the event $\Pred{growth}(a,0)$ occurs and leads to a successor situation $S_1$. The direct effect of $growth(a,0)$ is a change of a relation between $a$ and $b$ from $\Pred{rccNTPP}(a,b,0)$ to $\Pred{rccEQ}(a,b,1)$ (i.e., $a$ is equal to $b$). No change of the relation between $a$ and $c$ is directly stated, and thus we must derive the relation $\Pred{rccEC}(a,c,1)$ as an indirect effect. 

$\triangleright$ \emph{Motion.} Let $a$ move in step $S_0$; the event $\Pred{motion}(a,0)$ leads to a successor situation $S_1$. The direct effect is a change of the relation $\Pred{rccNTPP}(a,b,0)$ to $\Pred{rccTPP}(a,b,1)$ ($a$ is a tangential proper part of $b$). In the successor situation $S_1$ we must determine that the relation between $a$ and $c$ can only be either $\Pred{rccDC}(a,c,1)$ or $\Pred{rccEC}(a,c,1)$.

GQR provides no support for domain-specific reasoning, and thus we encoded the problem as two distinct qualitative constraint networks (one for each simulation step) and solved them independently, i.e., with no definition of \emph{growth} and \emph{motion}. Thus, GQR is not able to produce any additional information about indirect effects. As Clingo lacks any mechanism for analytic geometry, we implemented the RCC8 composition table and thus it inherits the incompleteness of relation algebraic reasoning. While CLP(QS) facilitates the modelling of domain rules such as \emph{growth}, there is no native support for default reasoning and thus we forced $b$ and $c$ to remain unchanged between simulation steps, otherwise all combinations of spatially consistent actions on $b$ and $c$ are produced without any preference (i.e., leading to the frame problem). 

In contrast, ASPMT(QS) can express spatial inertia, and derives indirect effects directly from spatial reasoning: in the \emph{Growth} problem we can use ASPMT(QS) to abduce that $a$ has to be concentric with $b$ in $S_0$ (otherwise a \emph{move} event would also need to occur).\footnote{That is, we can define a binary predicate $\Pred{concentric}$ between two circles that is \emph{true} when the centre points of the circles are equal,  and \emph{false} when they are not equal. ASPMT(QS) abduces that $\Pred{not concentric}$ is unsatisfiable, i.e. that the circles are necessarily concentric.} Checking global consistency of scenarios that contain interdependent spatial relations is a crucial feature that is enabled by the method of polynomial encodings and is provided only by CLP(QS) and ASPMT(QS).

\subsection{Geometric Reasoning and the Frame Problem}

In problems \emph{Attachment I} and \emph{Attachment II} the initial situation $S_0$ consists of three objects (circles), namely $\Const{car}$, $\Const{trailer}$ and $\Const{garage}$ as presented in Figure~\ref{fig::car}. Initially, the $\Const{trailer}$ is attached to the $\Const{car}$: $\Pred{rccEC}(\Const{car},\Const{trailer},0)$,  $\Pred{attached}(\Const{car},\Const{trailer},0)$. The successor situation $S_1$ is described by $\Pred{rccTPP}(\Const{car},\Const{garage},1)$. The task is to infer the possible relations between the trailer and the garage, and the necessary actions that would need to occur in each scenario.

There are two domain-specific actions: the car can move, $\Pred{move}(\Const{car},X)$, and the trailer can be detached, $\Pred{detach}(\Const{car},\Const{trailer},X)$ in simulation step $X$. Whenever the $\Const{trailer}$ is attached to the $\Const{car}$, they remain $\Pred{rccEC}$. The $\Const{car}$ and the $\Const{trailer}$ may be either completely outside or completely inside the $\Const{garage}$. 



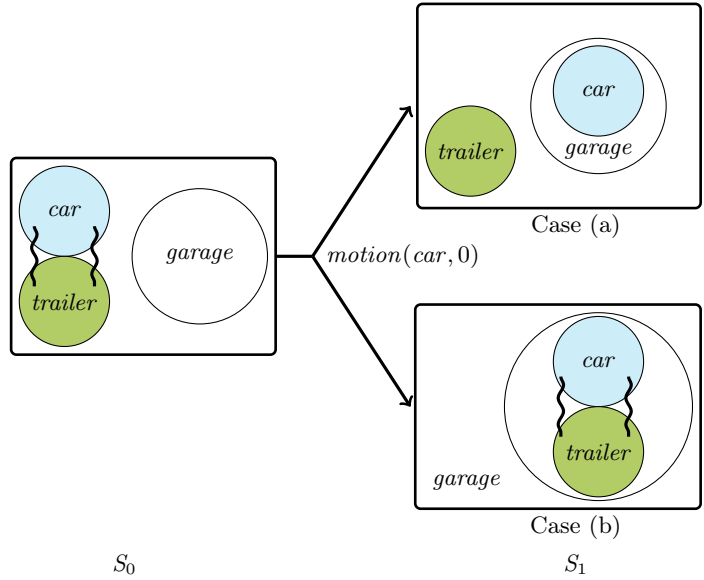
\begin{figure}[ht]
\begin{tikzpicture}


\node[draw=none,fill=none] at (1,-3.1)  {$S_0$};

\node (garage0) [draw=black, circle, minimum size=18mm, inner sep=0pt,outer sep=0] at (2,1) {$garage$};
\node (car0) [draw=black, fill=LightBlue, circle, minimum size=12mm, inner sep=0pt,outer sep=0] at (0.2,1.6) {$car$};
\node (trailer0) [draw=black, fill=LightGreen, circle, minimum size=12mm, inner sep=0pt,outer sep=0] at (0.2,0.4) {$trailer$};
\node[line width=1.0pt, rounded corners=2pt, fit=(garage0)(car0)(trailer0), draw] {};

\draw[-,very thick,decorate,decoration={snake,amplitude=1}] (-0.2,1.4) -- (-0.2,0.6);
\draw[-,very thick,decorate,decoration={snake,amplitude=1}] (0.6,1.4) -- (0.6,0.6);

\node[draw=none,fill=none] at (7,-3.1)  {$S_1$};

\node (garage3) [draw=white, circle, minimum size=25mm, inner sep=0pt,outer sep=0] at (7.3,3) {$ $};

\node (garage1) [draw=black, circle, minimum size=18mm, inner sep=0pt,outer sep=0] at (7.3,3) {$ $};
\node (car1) [draw=black, fill=LightBlue, circle, minimum size=12mm, inner sep=0pt,outer sep=0] at (7.3,3.2) {$car$};
\node (trailer1) [draw=black, fill=LightGreen, circle, minimum size=12mm, inner sep=0pt,outer sep=0] at (5.6,2.4) {$trailer$};
\node[line width=1.0pt, rounded corners=2pt, fit=(garage1)(car1)(trailer1)(garage3), draw] {};
\node[draw=none,fill=none] at (7.3,2.4)  {$garage$};


\node (garage1) [draw=black, circle, minimum size=25mm, inner sep=0pt,outer sep=0] at (7.3,-1) {$ $};
\node (car1) [draw=black, fill=LightBlue, circle, minimum size=12mm, inner sep=0pt,outer sep=0] at (7.3,-0.4) {$car$};
\node (trailer1) [draw=black, fill=LightGreen,  circle, minimum size=12mm, inner sep=0pt,outer sep=0] at (7.3,-1.6) {$trailer$};
\node (garage3) [draw=none,fill=none] at (5.55,-2.0) {$garage$};
\node[line width=1.0pt, rounded corners=2pt, fit=(garage1)(car1)(garage3)(trailer1), draw] {};

\draw[-,very thick,decorate,decoration={snake,amplitude=1}] (6.8,-0.6) -- (6.8,-1.4);
\draw[-,very thick,decorate,decoration={snake,amplitude=1}] (7.7,-0.6) -- (7.7,-1.4);

\draw[ line width=1.2pt, rounded corners=2pt] (3,1) -- (3.5,1);

\draw[->, line width=1.2pt, rounded corners=2pt] (3.5,1) -- (4.8,3);
\draw[->, line width=1.2pt, rounded corners=2pt] (3.5,1) -- (4.8,-1);

\node at (4.7,1) [] {$motion(car,0)$};

\node at (7,1.4) [] {Case (a)};
\node at (7,-2.6) [] {Case (b)};

\end{tikzpicture}
\caption{Non-monotonic reasoning with additional geometric information.}
\label{fig::car}
\end{figure}

$\triangleright$ \emph{Attachment I.} Given the available topological information, we must infer that there are two possible solutions (Figure.~\ref{fig::car}); (a) the $\Const{car}$ was detached from the $\Const{trailer}$ and then moved into the $\Const{garage}$:
(b) the $\Const{car}$, together with the $trailer$ attached to it, moved into the $\Const{garage}$:

\smallskip

$\triangleright$ \emph{Attachment II.} We are given additional geometric information about the objects' size: $r(car)=2$, $r(trailer)=2$ and $r(garage)=3$. Case (b) is now inconsistent, and we must determine that the only possible solution is (a).

These domain-specific rules require default reasoning: ``\emph{typically} the $\Const{trailer}$ remains in the same position'' and ``\emph{typically} the $\Const{trailer}$ remains attached to the $\Const{car}$''. The later default rule is formalised in ASPMT(QS) by means of the spatial default.
The formalisation of such rules addresses the frame problem. GQR is not capable of expressing the domain-specific rules for detachment and attachment in \emph{Attachment~I} and \emph{Attachment~II}. Neither GQR nor Clingo is capable of reasoning with a combination of topological and numerical information, as required in \emph{Attachment~II}. As CLP(QS) cannot express default rules, we can not capture the notion that, for example, the trailer should typically remain in the same position unless we have some explicit reason for determining that it moved; once again this leads to an exhaustive enumeration of all possible scenarios without being able to specify preferences, i.e. the frame problem, and thus CLP(QS) will not scale in larger scenarios.

The results of the empirical evaluation show that ASPMT(QS) is the only system that is capable of (a) non-monotonic spatial reasoning, (b) expressing domain-specific rules that also have spatial aspects, and (c) integrating both qualitative and numerical information. Regarding the greater execution times in comparison to CLP(QS), we have not yet fully integrated available optimisations (e.g., based on spatial symmetry-driven pruning strategies \cite{DBLP:conf/cosit/SchultzB15}) with respect to spatial reasoning; this is one of the directions of future work as the known performance gains are truly significant \cite{DBLP:conf/cosit/SchultzB15}.

\subsection{Abductive Reasoning}
We show how abductive reasoning may be achieved in ASPMT(QS). Consider an application where the spatial configuration of objects is recorded in discrete time points (e.g., geospatial information collected about cities or a dynamic environment observed by a mobile robot). We consider the situation presented in Figure \ref{fig::cells-nonmono}, where the following data about spatial relations between $a, b, c$ in time points  $t_1$, $t_2$, $t_3$ is available:

\includegraphics[]{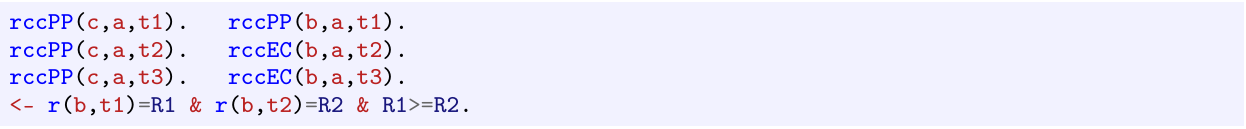}


\begin{figure}[t]
\centering
\includegraphics[width=0.8\textwidth]{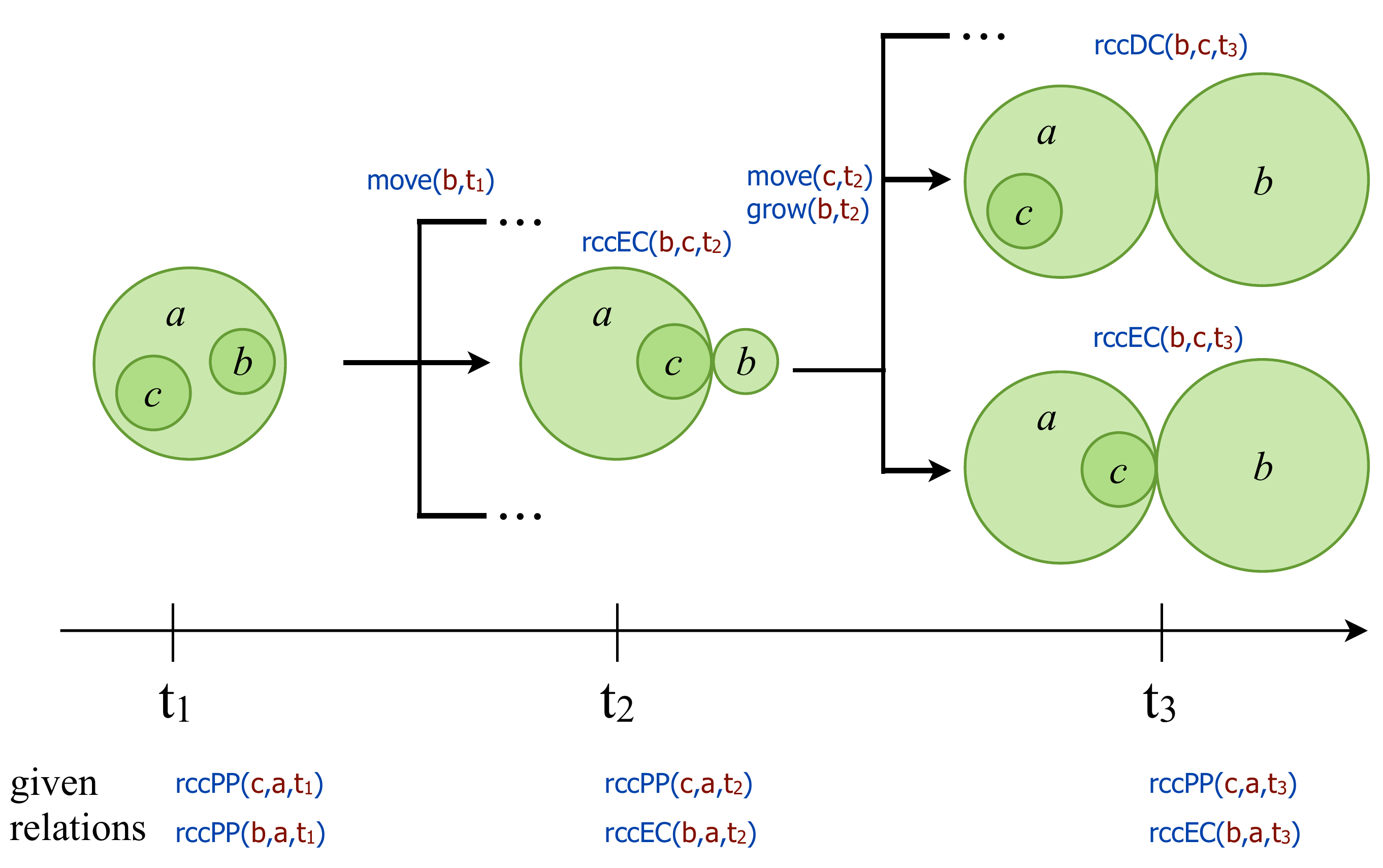}
\caption{Non-monotonically inferring actions and qualitative spatial relations}
\label{fig::cells-nonmono}
\end{figure}

\noindent Having a knowledge base about available actions that may be performed and a number of rules describing the spatial behaviour of objects, ASPMT(QS) infers what actions had to be performed between time points, namely:

\includegraphics[]{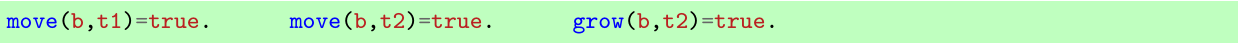}


Abductive reasoning requires default rules such as inertia which states that if no action is performed, then object locations and parameters remain the same by default. This rule enables the inference that the position of $c$ with respect to $a$ is the same in all three time steps. 
The example illustrated in Figure \ref{fig::cells-nonmono} shows that ASPMT(QS) checks the global consistency of various spatial relations. Therefore, it infers relations occurring as a result of performed actions (indirect effects) even if there is more than one possible solution, e.g., the relation between $c$ and $b$ in $t_3$ may be either $rccDC(b,c,t_3)$ or $rccEC(b,c,t_3)$.

\section{Discussion and Related Work}\label{sec:discussions-relwork}
Our research investigates the formal modelling and computational aspects of \emph{space} and \emph{spatio-temporal dynamics} \cite{bhatt:scc:08,Bhatt:RSAC:2012} from the viewpoint of the artificial intelligence (AI)  sub-disciplines of knowledge representation and reasoning (KR), commonsense reasoning, geometric and qualitative spatial reasoning, and more broadly, spatial information theory and spatial cognition and computation.

\smallskip

\noindent\mysubsecNoBook{{\color{black}Beyond the state of the art}}{XX}
Research in \emph{qualitative spatial representation and reasoning}(QSTR) has focussed on the construction of relation-algebraically founded qualitative spatial and temporal calculi, complexity results thereof, and formalisations of (topological) spatial logics. Furthermore, specialized \emph{relation-algebraically} founded methods and prototypical spatial reasoning tools -- \emph{working as black-box systems outside of any systematic KR method} -- have been a niche too. What is still missing in the field of spatial representation and reasoning is a modular, unifying KR framework of \emph{space, action, and change} \cite{Bhatt:RSAC:2012} that would seamlessly integrate with or be accessible  via general  KR languages and frameworks in AI, and be applicable in a wide-range of application domains. Primarily, our research is motivated by addressing this gap.

\smallskip

\noindent\mysubsecNoBook{{\color{black}KR based methods and tools addressing spatio-temporal dynamics}}{XX} 
With a focus on commonsense reasoning about space, action, and change, the {principal focus and long-term agenda of our research is to develop methods and generally usable tools} for visuo-spatial problem solving with spatio-temporal configurations and dynamic phenomena at the scale of everyday human perception, abstraction, interpretation, and interaction. This agenda is being driven by the application domains that are being pursued independently in the areas of \emph{spatio-temporal narrative interpretation and synthesis, spatial computing for architecture design, cognitive vision and robotics, and geospatial dynamics} \cite{Bhatt-Schultz-Freksa:2013}. This level of generality, and our emphasis on mixed qualitative-quantitative spatial reasoning in the context of state of the art KR methods, namely constraint logic and answer set programming paradigms, will inherently offer a robust and scalable representational and computational foundation for the class of application areas that inspire and guide the basic research questions reported in this paper.

\smallskip

%
%

\noindent\mysubsecNoBook{{\color{black}Related work on space and motion}}{XX}
\cite{DefSpaOccu:1995:Shanahan} describes a default reasoning problem, analogous to the classic frame problem, which arises when an attempt is made to construct a logic-based calculus for reasoning about the movement of objects in a real-valued co-ordinate system. \cite{DefSpaOccu:1995:Shanahan}'s all-encompassing theory alludes to a unification of spatial, temporal and causal aspects at representational and computational levels. The use of commonsense reasoning about the \emph{physical properties of objects} within a first-order logical framework has been investigated by \cite{davis:08-liquid,davis:09-box}. The aim here is to combine commonsense qualitative reasoning about ``continuous time, Euclidean space, commonsense dynamics of solid objects, and semantics of partially specified plans'' \cite{davis:09-box}. \cite{Cabalar2010} investigate the formalization of the commonsense representation that is necessary to solve \emph{spatial puzzles} involving non-trivial objects such as \emph{holes} and \emph{strings}. \cite{bhatt:scc:08} explicitly formalize a \emph{dynamic spatial systems} approach for the modelling of changing spatial domains using the Situation Calculus \cite{SitCalc:McCarthy:69}. A dynamic spatial system here is regarded as a specialization of the generic \emph{dynamic systems} approach \cite{dynamicSys:Sandewall:1994,KnowInAction:2001:Reiter} for the case where sets of qualitative spatial relationships (grounded in formal spatial calculi) undergo change as a result of actions and events in the system. 




\section{Conclusion and Outlook}\label{sec::conclusions}
We have presented ASPMT(QS),  a novel approach for reasoning about space and spatial change within a KR paradigm. By integrating dynamic spatial reasoning within a KR framework, namely answer set programming (modulo theories), our system can be used to model behaviour patterns that characterise high-level processes, events, and activities as identifiable with respect to a general characterisation of commonsense \emph{reasoning about space, actions, and change} \cite{Bhatt:RSAC:2012,bhatt:scc:08}. 

ASPMT(QS) is capable of sound and complete spatial reasoning, and combining qualitative and quantitative spatial information when reasoning non-monotonically; this is due to the approach of encoding spatial relations as polynomial constraints, and solving using SMT solvers with the theory of real nonlinear arithmetic. We have demonstrated that no other existing spatial reasoning system is capable of supporting the non-monotonic spatial reasoning features (e.g., spatial inertia, ramification, causal explanation) in the context of any systematic knowledge representation and reasoning method, be it a mainstream method such as answer set programming, or otherwise.

This work opens up several opportunities: the spatio-temporal ontology can be extended in many interesting ways with support for richer spatial relational as well as object domains. The polynomial encodings may be further optimised, and this is a topic that we have ourselves devoted considerable attention to \cite{DBLP:conf/cosit/SchultzB15} (Appendix H). Most interestingly, there exist many possibilities to build additional modules directly on top of ASPMT(QS) concerning aspects such as visibility, spatio-temporal motion patterns (based on space-time histories) etc from the viewpoint of applications in cognitive vision, cognitive robotics, eye-tracking or visual perception \cite{SuchanBS14-perceptualnarrative,ijcai2016-vision-eye-film,wacv2016-vision-eye-film}, and to provide support for specialized, domain-specific spatio-linguistic characterisations identifiable in a range of spatial assistance systems and assistive technologies encompassing (spatial) logic, (spatial) language, and (spatial) cognition \cite{Bhatt-Schultz-Freksa:2013}.

\subsubsection*{Acknowledgments.} 


We acknowledge the support provided by the Polish National Science Centre grant 2011/02/A/HS1/0039, and the DesignSpace Group (\url{www.designspace.org}).

\medskip
\medskip

{
\nocite{ijcai2016-language-suchan,ijcai2016-vision-eye-film,aspmtqs-lpnmr-2015}
\bibliography{aspmtqs,MehulBibs/NMSR,MehulBibs/ijcai15,MehulBibs/QSR-bib,MehulBibs/RAC-bib,MehulBibs/H4-External,MehulBibs/H4-Internal,MehulBibs/P3-External,MehulBibs/P3-Internal,MehulBibs/ASP}

}

\end{document}